\documentclass[10pt]{article} 

\usepackage[preprint]{rlj} 

\usepackage{amssymb}            
\usepackage{mathtools}          
\usepackage{mathrsfs}           
\usepackage{graphicx}           
\usepackage{subcaption}         
\usepackage[space]{grffile}     
\usepackage{url}                
\usepackage{lipsum}             

\usepackage{booktabs}
\usepackage{arydshln}
\usepackage[table]{xcolor}
\usepackage[utf8]{inputenc} 
\usepackage[T1]{fontenc}    
\usepackage{hyperref}       
\usepackage{url}            
\usepackage{booktabs}       
\usepackage{amsfonts}       
\usepackage{nicefrac}       
\usepackage{microtype}      
\usepackage{xcolor}         

\usepackage{graphicx}
\usepackage{amsmath} 
\usepackage{amsthm}

\newtheorem{proposition}{Proposition}
\newtheorem{assumption}{Assumption}
\newtheorem{lemma}{Lemma}
\newtheorem{theorem}{Theorem}

\title{On the Convergence of Semi-Gradient TD:\\ A Mean Field Perspective}

\setrunningtitle{Enter Your Running Title Here}

\author{\name Haruki Settai \email sharuki692@g.ecc.u-tokyo.ac.jp \\
       \addr The University of Tokyo
       \AND
       \name Tadashi Kozuno \email  tadashi.kozuno@gmail.com\\
       \addr Sun Intelligence
       \AND
       \name Shinji Ito \email shinji@mist.i.u-tokyo.ac.jp \\
       \addr The University of Tokyo \& RIKEN}

\begin{document}

\maketitle  

\begin{abstract}
Semi-gradient temporal-difference learning is a core algorithmic component of reinforcement learning, but bootstrapping with function approximation generally induces a non-conservative vector field and can lead to instability. We analyze this effect for two-layer neural networks in the mean-field regime and identify a simple stability mechanism in which entropy-induced relaxation dominates the sensitivity of the moving bootstrap target. Under this condition, we prove exponential convergence to a unique fixed point and quantify its bias relative to the unregularized Bellman solution. We also show that the same mechanism extends to entropy-regularized off-policy control through soft Q-learning.
\end{abstract}

\section{Introduction}
Temporal-difference (TD) learning is one of the basic mechanisms for value estimation in reinforcement learning. Its semi-gradient form is particularly simple, updating the current prediction toward a bootstrapped target without differentiating through that target. This simplicity, however, also makes its dynamics fundamentally different from ordinary gradient descent. With function approximation, the expected semi-gradient update generally defines a non-conservative vector field, and the resulting dynamics can diverge or oscillate even in the absence of sampling noise~\citep{baird:95,gordon:00,asadi:23}.

Understanding this instability becomes particularly challenging with neural-network function approximation, where the representation itself evolves during learning. Despite the practical importance of neural TD methods, convergence theory for semi-gradient TD beyond linear or near-initialization regimes remains limited. The mean-field regime provides a tractable framework for studying feature learning in wide two-layer neural networks by describing training through the evolution of a distribution over parameters. Yet mean-field analyses of semi-gradient TD remain scarce. \citet{zhang:20} establish a best-iterate guarantee in function value, but do not prove convergence of the learning dynamics themselves. \citet{agazzi:22} characterize the equilibrium structure of mean-field semi-gradient TD and show that a convergent trajectory induces the true value function under suitable assumptions, while their convergence guarantee is obtained in a lazy-training regime. Thus, convergence of nonlinear mean-field semi-gradient TD dynamics, and the mechanism that can stabilize them, remain poorly understood.

In this work, we study noisy semi-gradient TD with two-layer neural networks in the mean-field regime. We write $q_t(\theta)d\theta$ for the parameter law evolving under the mean-field semi-gradient Langevin dynamics. Following the changing-objective viewpoint of \citet{asadi:23}, for any parameter law $q$, let $F(q,\nu)$ denote the frozen-target regression objective with entropy and quadratic regularization, where the bootstrap target is frozen at $q$ and $\nu$ represents the online predictor. We define the corresponding frozen-target minimizer map by
$$
T(q):=\arg\min_{\nu}F(q,\nu).
$$
If the bootstrap target were fixed, mean-field Langevin theory provides relaxation toward $T(q)$ \citep{chizat:22,nitanda:22}. In semi-gradient TD, however, the target is generated by the current parameter law itself, so the relevant minimizer is $T(q_t)$ and moves together with $q_t$. The central question is whether relaxation toward the current frozen-target minimizer is strong enough to dominate this moving-target effect.

Our first contribution is a last-iterate convergence guarantee for
mean-field semi-gradient TD. The analysis is governed by the relaxation rate $\rho\tau$ induced by noise and regularization and the constant $L_p$ measuring the sensitivity of the frozen-target objective to changes in the bootstrap target. If $\rho\tau>L_p$, then the frozen-target optimality gap
$$
G(t):=F(q_t,q_t)-F(q_t,T(q_t))
$$
decays exponentially as
$$
G(t)\leq e^{-2(\rho\tau-L_p)t}G(0).
$$
Under the same condition, the map $T$ is $L_p/(\rho\tau)$-contractive in Wasserstein distance and therefore admits a unique fixed point $q_\star=T(q_\star)$. Combining the decay of the frozen-target optimality gap with this contraction yields exponential convergence of the actual semi-gradient parameter law $q_t$ to $q_\star$ (Theorem~\ref{thm:policy_evaluation_convergence}). As summarized in Table~\ref{tab:comparison}, previous results establish either feature learning or last-iterate convergence, but not both simultaneously. We provide the first last-iterate convergence guarantee for mean-field semi-gradient TD.

The limiting law $q_\star$, however, solves an entropy- and
quadratically-regularized problem rather than the original Bellman
equation. We therefore complement the dynamical convergence result with
an explicit bound on the resulting approximation and regularization
error (Theorem~\ref{thm:policy_evaluation_error}). In the realizable
case, suppose there exists a parameter law $q^\pi$ such that
$V_{q^\pi}=V^\pi$ and
$D_{\mathrm{KL}}(q^\pi\|g_{\lambda,\tau})<\infty$. Then
$$
\|V_{q_\star}-V^\pi\|_{L^2(d^\pi)}
\leq
\sqrt{
\frac{\tau}{1-\gamma}
D_{\mathrm{KL}}
(q^\pi\|g_{\lambda,\tau})
}.
$$
Thus, the same regularization that stabilizes the learning dynamics introduces a quantitatively controlled bias in the learned value function. Our general result also allows for function-approximation error.

Finally, we extend the same stability mechanism to entropy-regularized off-policy soft Q-learning, obtaining exponential convergence to a unique self-consistent equilibrium (Theorem~\ref{thm:soft_q_convergence}). We also quantify the error of this equilibrium relative to the soft Bellman fixed point (Theorem~\ref{thm:soft_q_error}).

\begin{table}[t]
    \centering
    \caption{Comparison of convergence results for nonlinear TD learning.}
    \label{tab:comparison}
    \small
    \begin{tabular}{lcccc}
        \toprule
        & Neural function
        & Feature
        & Last-iterate
        & Control \\
        & approximation
        & learning
        & convergence
        & extension \\
        \midrule
        Cai et al.\ (2019)
        & \checkmark
        & --
        & --
        & \checkmark \\

        Zhang et al.\ (2020)
        & \checkmark
        & \checkmark
        & --
        & \checkmark \\

        Agazzi and Lu (2022)
        & \checkmark
        & --
        & \checkmark
        & -- \\

        Asadi et al.\ (2023)
        & --
        & --
        & \checkmark
        & \checkmark \\

        \textbf{Ours}
        & \checkmark
        & \checkmark
        & \checkmark
        & \checkmark \\
        \bottomrule
    \end{tabular}

    \vspace{2pt}
    \begin{minipage}{0.98\linewidth}
    \footnotesize
    \textit{Notes.}
    The entries indicate guarantees established within the respective settings
    of each work and should not be interpreted as implication or dominance
    relations between the results.
    In particular, our convergence guarantees are obtained for regularized
    dynamics and therefore introduce bias relative to the unregularized Bellman
    solution. We quantify this regularization bias separately in
    Theorems~\ref{thm:policy_evaluation_error} and~\ref{thm:soft_q_error}.
    \end{minipage}
\end{table}

\paragraph{Additional related work.}
Beyond the mean-field setting, convergence of TD learning with nonlinear or neural function approximation has been studied under various structural and overparameterization assumptions. \citet{brandfonbrener:20} analyze the expected nonlinear TD dynamics through geometric properties of the function class and the underlying Markov chain. A separate line of work establishes finite-time guarantees for neural TD using overparameterization and control of the network around its initialization \citep{cai:19,cayci:23,ke:24}. These results provide complementary convergence guarantees outside the mean-field feature-learning regime considered here.

\section{Problem Setting}

\subsection{Notation}

Let $\mathcal P_2(\mathbb R^d)$ denote the space of probability measures on
$\mathbb R^d$ with finite second moment, and let
$\mathcal P_2^{\mathrm{ac}}(\mathbb R^d)$ denote its subset of measures
absolutely continuous with respect to Lebesgue measure. We write $W_2$
for the 2-Wasserstein distance. For $\nu_i(\mathrm{d}\theta)=q_i(\theta)\mathrm{d}\theta
\in \mathcal P_2^{\mathrm{ac}}(\mathbb R^d)$, $i=1,2$, we define the
differential entropy, the Kullback--Leibler divergence, and the relative
Fisher information by
$$
\mathcal H(\nu_1)
:= -\int q_1 \log q_1\,\mathrm{d}\theta,
\qquad
D_{\mathrm{KL}}(\nu_1\|\nu_2)
:= \int q_1 \log\frac{q_1}{q_2}\,\mathrm{d}\theta,
$$
and
$$
I(\nu_1\|\nu_2)
:= \int q_1
\left\|\nabla\log\frac{q_1}{q_2}\right\|^2 \mathrm{d}\theta.
$$
We identify an absolutely continuous probability measure with its density
whenever no confusion arises.

\subsection{Reinforcement Learning Setup}

We consider a discrete-time discounted Markov decision process
$(\mathcal S,\mathcal A,P,r,\gamma)$, where $\mathcal S$ and $\mathcal A$
are possibly infinite bounded subsets of Euclidean spaces.
Here $P$ is the transition kernel,
$r:\mathcal S\times\mathcal A\to[0,1]$ is the reward function,
and $\gamma\in(0,1)$ is the discount factor.
A policy $\pi:\mathcal S\to\Delta(\mathcal A)$ maps each state
$s\in\mathcal S$ to a distribution $\pi(\cdot|s)$ over $\mathcal A$.

We first consider the policy evaluation problem, where a policy $\pi$ is
fixed and the goal is to learn its value function $V^\pi$, namely the
expected infinite-horizon discounted return under $\pi$.
It is the unique solution of the Bellman equation
$$
V^\pi(s)
=
\mathbb E_{\substack{a\sim\pi(\cdot|s)\\
s'\sim P(\cdot|s,a)}}
\left[
r(s,a)+\gamma V^\pi(s')
\right].
$$

We later extend the same framework to entropy-regularized off-policy
control through soft Q-learning.

\subsection{Semi-Gradient TD}

Let $V(s;\vartheta)$ be a value-function approximator with parameter
$\vartheta$. Given a transition $(s_k,a_k,s_{k+1})$, semi-gradient TD
updates
$$
\vartheta^{(k+1)}
=
\vartheta^{(k)}
+
\alpha \delta_k
\nabla_\vartheta V(s_k;\vartheta^{(k)}),
$$
where
$$
\delta_k
=
r(s_k,a_k)
+
\gamma V(s_{k+1};\vartheta^{(k)})
-
V(s_k;\vartheta^{(k)}).
$$
Unlike gradient descent on the squared TD error, this update does not
differentiate through the bootstrapped term
$V(s_{k+1};\vartheta^{(k)})$. Consequently, its expected update field
need not be the gradient of any fixed scalar objective and may exhibit
oscillation or divergence under function approximation
\citep{baird:95,tsitsiklis:97,gordon:00}.

\subsection{Two-Layer Neural Networks and Mean-Field Formulation}
\label{subsec:two_layer}

We now specialize the value-function approximator to a two-layer neural
network and formulate its semi-gradient dynamics in the mean-field
parameterization. Let $\mathcal X$ denote the input space and
$\Theta=\mathbb R^d$ the parameter space of a single neuron. For policy
evaluation, $\mathcal X=\mathcal S$. We consider the feature map
$$
\phi(x;\theta)
=
\sigma_1(c)\sigma_2(w^\top x+b),
\qquad
\theta=(c,w,b)\in\Theta,
$$
where $\sigma_1,\sigma_2\in C^1$ have bounded values and derivatives.
Given particles
$\boldsymbol\theta=(\theta_1,\ldots,\theta_M)\in\Theta^M$, define the
two-layer network
$$
V_M(s;\boldsymbol\theta)
:=
\frac{1}{M}\sum_{i=1}^M \phi(s;\theta_i).
$$

Applying the semi-gradient TD update from the previous section to this
network gives, for each particle,
$$
\theta_i^{(k+1)}
=
\theta_i^{(k)}
+
\frac{\alpha}{M}\,
\delta_k
\nabla_\theta\phi(s_k;\theta_i^{(k)}),
$$
where
$$
\delta_k
=
r(s_k,a_k)
+
\gamma V_M(s_{k+1};\boldsymbol\theta^{(k)})
-
V_M(s_k;\boldsymbol\theta^{(k)}).
$$
Writing $\eta:=\alpha/M$, the particle motion is therefore the usual
semi-gradient TD update expressed in the two-layer parameterization.

To obtain the mean-field representation, introduce the empirical
parameter law
$$
\nu_M
:=
\frac{1}{M}\sum_{i=1}^M\delta_{\theta_i}.
$$
Then
$$
V_M(s;\boldsymbol\theta)
=
V(s;\nu_M),
\qquad
V(s;\nu)
:=
\int_\Theta
\phi(s;\theta)\,\nu(\mathrm d\theta),
$$
for any $\nu\in\mathcal P_2(\Theta)$. The mean-field formulation replaces
the finite empirical law $\nu_M$ by a general probability law
$\nu\in\mathcal P_2(\Theta)$.

To represent the semi-gradient structure at the level of probability
measures, we separate the law generating the bootstrap target from the
law being updated. Following the changing-objective viewpoint of
\citet{asadi:23}, for a target law $\mu$ and an online law $\nu$, define
$$
H(\mu,\nu)
:=
\frac{1}{2}
\mathbb E_{s,a,s'}
\left[
\left(
r(s,a)
+
\gamma V(s';\mu)
-
V(s;\nu)
\right)^2
\right].
$$
The expectation is taken with respect to the transition distribution
used by the algorithm. For fixed $\mu$, the functional $H(\mu,\cdot)$
is an ordinary regression objective toward the bootstrap target
generated by $\mu$. Semi-gradient TD differentiates only with respect
to the online law and evaluates the two arguments on the diagonal
$\mu=\nu$. Thus, while $H(\mu,\cdot)$ is a fixed objective when $\mu$
is held fixed, the objective itself moves along the actual
semi-gradient dynamics.

We study the regularized dynamics obtained by adding the quadratic
functional
$$
R(\nu)
:=
\frac{\lambda}{2}
\int_\Theta
\|\theta\|^2\,\nu(\mathrm d\theta),
\qquad
\lambda>0,
$$
and define
$$
E(\mu,\nu)
:=
H(\mu,\nu)+R(\nu).
$$
For a functional of a probability law, we write
$\delta/\delta\nu$ for its first variation with respect to $\nu$.
In the present setting,
$$
\frac{\delta H}{\delta\nu}(\mu,\nu)(\theta)
=
-
\mathbb E_{s,a,s'}
\left[
\left(
r(s,a)
+
\gamma V(s';\mu)
-
V(s;\nu)
\right)
\phi(s;\theta)
\right],
$$
and hence
$$
\frac{\delta E}{\delta\nu}(\mu,\nu)(\theta)
=
\frac{\delta H}{\delta\nu}(\mu,\nu)(\theta)
+
\frac{\lambda}{2}\|\theta\|^2.
$$
Taking the parameter gradient of the first variation and evaluating on
the diagonal $\mu=\nu$ recovers the expected particle-wise
semi-gradient direction, with the additional drift induced by the
quadratic regularization. The diagonal evaluation is the key
distinction from ordinary optimization of the fixed functional
$H(\mu,\cdot)$.

\subsection{Noisy Semi-Gradient Mean-Field Dynamics}
\label{subsec:noisy_dynamics}

We now add isotropic Gaussian noise to the population semi-gradient
dynamics. At the particle level, the resulting update is
\begin{equation}
\theta_i^{(k+1)}
=
\theta_i^{(k)}
-
\eta
\nabla_\theta
\frac{\delta E}{\delta\nu}
\bigl(
\nu_M^{(k)},\nu_M^{(k)}
\bigr)
\bigl(
\theta_i^{(k)}
\bigr)
+
\sqrt{2\tau\eta}\,\xi_i^{(k)},
\label{eq:noisy_particle_update}
\end{equation}
where $\tau>0$ controls the noise level,
$\xi_i^{(k)}\sim\mathcal N(0,I_d)$ are independent, and
$$
\nu_M^{(k)}
=
\frac{1}{M}\sum_{i=1}^M
\delta_{\theta_i^{(k)}}
$$
is the empirical parameter law. Here the TD drift is the expected
semi-gradient drift; the stochasticity in
\eqref{eq:noisy_particle_update} comes from the added Langevin
perturbation.

In continuous time, the corresponding mean-field dynamics are described
by the McKean--Vlasov diffusion
\begin{equation}
\mathrm d\theta_t
=
-
\nabla_\theta
\frac{\delta E}{\delta\nu}
(\nu_t,\nu_t)(\theta_t)\,\mathrm dt
+
\sqrt{2\tau}\,\mathrm dW_t,
\label{eq:mckean_vlasov}
\end{equation}
where $\nu_t$ denotes the law of $\theta_t$.
Under suitable regularity and growth conditions on the drift, standard
Fokker--Planck--Kolmogorov theory ensures propagation of finite second
moments and absolute continuity of the time-marginal laws
\citep{bogachev:15}. We therefore write
$$
\nu_t(\mathrm d\theta)
=
q_t(\theta)\,\mathrm d\theta,
\qquad
\nu_t\in\mathcal P_2^{\mathrm{ac}}(\mathbb R^d).
$$
The density $q_t$ evolves according to the nonlinear Fokker--Planck
equation
\begin{equation}
\partial_t q_t
=
\nabla_\theta\cdot
\left(
q_t
\nabla_\theta
\frac{\delta E}{\delta\nu}(q_t,q_t)
\right)
+
\tau\Delta_\theta q_t.
\label{eq:fokker_planck}
\end{equation}

Define the entropy-regularized frozen-target functional
$$
F(\mu,\nu)
:=
E(\mu,\nu)-\tau\mathcal H(\nu),
$$
where $\mathcal H$ denotes the differential entropy defined in
Section~2.1. Since
$$
\nabla_\theta
\frac{\delta(-\tau\mathcal H)}{\delta\nu}(\nu)(\theta)
=
\tau\nabla_\theta\log q(\theta)
$$
for $\nu(\mathrm d\theta)=q(\theta)\mathrm d\theta$, the
Fokker--Planck equation can equivalently be written as
\begin{equation}
\partial_t q_t
=
\nabla_\theta\cdot
\left(
q_t
\nabla_\theta
\frac{\delta F}{\delta\nu}(q_t,q_t)
\right).
\label{eq:fokker_planck_F}
\end{equation}

For a fixed target law $\mu$, the dynamics associated with
$F(\mu,\cdot)$ are ordinary mean-field Langevin dynamics and relax
toward the minimizer of this fixed functional. In the actual
semi-gradient dynamics, however, the target law is the current law
$q_t$ itself and therefore moves together with the online law.
Consequently, \eqref{eq:fokker_planck_F} has a gradient-flow-like form
but is not the Wasserstein gradient flow of the diagonal functional
$\nu\mapsto F(\nu,\nu)$, since the first variation is taken only with
respect to the second argument.

\section{Main Assumptions and Preliminary Results}

We first introduce the two ingredients that govern the convergence
analysis. The first controls the coupling between the target and online
laws in the frozen-target objective. The second quantifies the
relaxation properties of the corresponding Gibbs distributions.

\begin{assumption}[Mixed smoothness of the frozen-target objective]
\label{ass:mixed_smoothness}
There exists $L_p<\infty$ such that, for all admissible
$\mu,\nu_1,\nu_2\in\mathcal P_2(\mathbb R^d)$ and every
$\theta\in\mathbb R^d$,
$$
\left\|
\nabla_\theta
\left[
\frac{\delta H}{\delta\mu}(\mu,\nu_1)(\theta)
-
\frac{\delta H}{\delta\mu}(\mu,\nu_2)(\theta)
\right]
\right\|
\leq
L_p W_2(\nu_1,\nu_2).
$$
\end{assumption}

The constant $L_p$ measures the mixed sensitivity of the frozen-target
objective to the interaction between the target and online laws. For the
two-layer parameterization in Section~\ref{subsec:two_layer},
Assumption~\ref{ass:mixed_smoothness} follows from the boundedness of
the feature map and its first derivatives together with a finite second
moment of the input distribution.

For a fixed target law $\mu$, the functional $F(\mu,\cdot)$ is convex
in the probability law and includes both quadratic and entropy
regularization. Its minimizer therefore admits a Gibbs representation.

\begin{proposition}[Gibbs form of the frozen-target minimizer; cf. \citealp{hu:21}]
\label{prop:gibbs_minimizer}
For every $\mu\in\mathcal P_2(\mathbb R^d)$, the functional
$\nu\mapsto F(\mu,\nu)$ admits a unique minimizer
$\pi_{\mu,*}\in\mathcal P_2^{\mathrm{ac}}(\mathbb R^d)$.
Its density satisfies
$$
\pi_{\mu,*}(\theta)
\propto
\exp\left(
-\frac{1}{\tau}
\frac{\delta E}{\delta\nu}
(\mu,\pi_{\mu,*})(\theta)
\right).
$$
\end{proposition}

More generally, for any
$\mu,\nu\in\mathcal P_2(\mathbb R^d)$, define the associated Gibbs law
$\pi_{\mu,\nu}$ by
$$
\pi_{\mu,\nu}(\mathrm d\theta)
=
\frac{1}{Z_{\mu,\nu}}
\exp\left(
-\frac{1}{\tau}
\frac{\delta E}{\delta\nu}(\mu,\nu)(\theta)
\right)
\mathrm d\theta,
$$
where $Z_{\mu,\nu}$ is the normalizing constant.
The quadratic regularizer together with the bounded perturbation induced
by $H$ ensures that $\pi_{\mu,\nu}$ is well defined and belongs to
$\mathcal P_2^{\mathrm{ac}}(\mathbb R^d)$.

The second ingredient is a uniform functional inequality for this Gibbs
family.

\begin{assumption}[Uniform log-Sobolev inequality]
\label{ass:lsi}
There exists $\rho>0$ such that, uniformly over all admissible
$\mu,\nu\in\mathcal P_2(\mathbb R^d)$, the Gibbs law
$\pi_{\mu,\nu}$ satisfies
$$
D_{\mathrm{KL}}(m\|\pi_{\mu,\nu})
\leq
\frac{1}{2\rho}
I(m\|\pi_{\mu,\nu})
$$
for every
$m\in\mathcal P_2^{\mathrm{ac}}(\mathbb R^d)$.
\end{assumption}

Our convergence results use Assumption~\ref{ass:lsi} only through the
effective relaxation strength $\rho\tau$ and do not require any specific
dependence of $\rho$ on the temperature $\tau$. Sufficient conditions for Assumption~\ref{ass:lsi} are given in
Appendix~\ref{app:lsi_verification}.

\begin{proposition}[Entropy sandwich; cf. \citealp{chizat:22,nitanda:22}]
\label{prop:entropy_sandwich}
For every $\mu\in\mathcal P_2(\mathbb R^d)$ and
$\nu\in\mathcal P_2^{\mathrm{ac}}(\mathbb R^d)$,
$$
\tau
D_{\mathrm{KL}}(\nu\|\pi_{\mu,*})
\leq
F(\mu,\nu)-F(\mu,\pi_{\mu,*})
\leq
\tau
D_{\mathrm{KL}}(\nu\|\pi_{\mu,\nu}).
$$
\end{proposition}

Together with Assumption~\ref{ass:lsi}, the entropy-sandwich inequality relates the rate of decrease along the frozen-target dynamics to the frozen-target optimality gap.

\section{Main Results}
\label{sec:main_results}

\subsection{Policy Evaluation}
\label{subsec:policy_evaluation_main}

We first establish convergence of the noisy mean-field semi-gradient
dynamics in the policy-evaluation setting. Throughout this section, we
write
$$
\nu_t(\mathrm d\theta)
=
q_t(\theta)\mathrm d\theta
$$
for the parameter law evolving according to
\eqref{eq:fokker_planck_F}. For each target law $\mu$, define the Gibbs response map
$$
T(\mu)
:=
\pi_{\mu,*}(\theta)\mathrm d\theta,
$$
where $\pi_{\mu,*}$ is the unique minimizer of
$\nu\mapsto F(\mu,\nu)$.

A direct analysis of the diagonal functional
$F(\nu_t,\nu_t)$ does not yield a closed estimate because the bootstrap
target moves together with the online law. We instead consider the
frozen-target optimality gap
$$
G(t)
:=
F(\nu_t,\nu_t)
-
F(\nu_t,T(\nu_t)).
$$
The first lemma shows that this gap decreases exponentially whenever
the relaxation rate $\rho\tau$ dominates the sensitivity $L_p$ of
the moving target.

\begin{lemma}[Decay of the frozen-target optimality gap]
\label{lem:gap_decay}
Suppose that
$\nu_0\in\mathcal P_2^{\mathrm{ac}}(\mathbb R^d)$ satisfies
$F(\nu_0,\nu_0)<\infty$, and let $(\nu_t)_{t\geq0}$ be a
sufficiently regular solution to \eqref{eq:fokker_planck_F}.
Under Assumptions~\ref{ass:mixed_smoothness} and~\ref{ass:lsi},
suppose further that $\rho\tau>L_p$.
Then
$$
\frac{d}{dt}G(t)
\leq
-2(\rho\tau-L_p)G(t).
$$
Consequently,
$$
G(t)
\leq
e^{-2(\rho\tau-L_p)t}G(0)
$$
for every $t\geq0$.
\end{lemma}

Lemma~\ref{lem:gap_decay} shows that the current parameter law
approaches its instantaneous frozen-target minimizer. This alone does
not imply convergence to a common limit because the minimizer itself
changes with the target law. The next lemma controls this dependence.

\begin{lemma}[Stability of the Gibbs response map]
\label{lem:response_contraction}
Under Assumptions~\ref{ass:mixed_smoothness} and~\ref{ass:lsi}, for
every $\mu_1,\mu_2\in\mathcal P_2(\mathbb R^d)$,
$$
W_2\bigl(T(\mu_1),T(\mu_2)\bigr)
\leq
\frac{L_p}{\rho\tau}
W_2(\mu_1,\mu_2).
$$
In particular, if $\rho\tau>L_p$, then $T$ is a contraction on
$\bigl(\mathcal P_2(\mathbb R^d),W_2\bigr)$.
\end{lemma}

Since $\bigl(\mathcal P_2(\mathbb R^d),W_2\bigr)$ is complete,
Lemma~\ref{lem:response_contraction} and Banach's fixed-point theorem
imply that, under $\rho\tau>L_p$, the response map $T$ admits a unique
fixed point. We denote it by $\nu_\star$, so that
$$
\nu_\star
=
T(\nu_\star).
$$
The two preceding lemmas now yield convergence of the full
semi-gradient dynamics.

\begin{theorem}[Exponential convergence to the self-consistent equilibrium]
\label{thm:policy_evaluation_convergence}
Suppose that the conditions of Lemma~\ref{lem:gap_decay} hold and that
$$
\rho\tau>L_p.
$$
Then $(\nu_t)_{t\geq0}$ converges exponentially in $W_2$ to the unique
self-consistent equilibrium $\nu_\star$. More precisely, for every
$t\geq0$,
$$
W_2(\nu_t,\nu_\star)
\leq
\frac{\sqrt{2\rho\tau}}{\rho\tau-L_p}
\sqrt{G(0)}
e^{-(\rho\tau-L_p)t}.
$$
\end{theorem}

Lemma~\ref{lem:gap_decay} controls relaxation toward the instantaneous
frozen-target minimizer, while
Lemma~\ref{lem:response_contraction} controls the sensitivity of that
minimizer to changes in the bootstrap target. The condition
$\rho\tau>L_p$ ensures that the relaxation is strong enough to dominate
the moving-target effect.

\subsection{Approximation and Regularization Error}
\label{subsec:policy_evaluation_error}

Theorem~\ref{thm:policy_evaluation_convergence} establishes convergence
to a unique self-consistent equilibrium. Since this equilibrium is
defined through an entropy- and quadratically-regularized objective,
it need not represent the unregularized value function exactly. We next
quantify this error.

Let
$$
g_{\lambda,\tau}
:=
\mathcal N\left(0,\frac{\tau}{\lambda}I_d\right).
$$
For any absolutely continuous parameter law $\nu$,
$$
R(\nu)-\tau\mathcal H(\nu)
=
\tau D_{\mathrm{KL}}(\nu\|g_{\lambda,\tau})
-
\frac{\tau d}{2}
\log\frac{2\pi\tau}{\lambda}.
$$
We write $V_\nu(s):=V(s;\nu)$. For any measurable function $f$ on $\mathcal S$, define the Bellman
operator of the evaluated policy by
$$
(T^\pi f)(s)
:=
\mathbb E_{\substack{
a\sim\pi(\cdot\mid s)\\
s'\sim P(\cdot\mid s,a)
}}
\left[
r(s,a)+\gamma f(s')
\right].
$$

\begin{theorem}[Approximation and regularization error]
\label{thm:policy_evaluation_error}
Let $P^\pi$ be the transition operator of the evaluated policy and let
$d^\pi$ be a stationary state distribution for $P^\pi$. Suppose that
the expectation defining the frozen-target objective $H$ is taken under
the transition distribution $d^\pi(ds)\,\pi(da\mid s)\,P(ds'\mid s,a)$.
Let $\nu_\star$ be a self-consistent equilibrium satisfying
$$
\nu_\star
\in
\operatorname*{argmin}_{\nu}
F(\nu_\star,\nu).
$$
Then
$$
\|V_{\nu_\star}-V^\pi\|_{L^2(d^\pi)}
\leq
\inf_{\nu}
\left\{
\frac{1+\gamma}{1-\gamma}
\|V_\nu-V^\pi\|_{L^2(d^\pi)}
+
\sqrt{
\frac{\tau}{1-\gamma}
D_{\mathrm{KL}}(\nu\|g_{\lambda,\tau})
}
\right\},
$$
where the infimum is over parameter laws with finite
$D_{\mathrm{KL}}(\nu\|g_{\lambda,\tau})$.

In particular, suppose that there exists a parameter law $\nu^\pi$
satisfying
$$
V_{\nu^\pi}=V^\pi,
\qquad
D_{\mathrm{KL}}(\nu^\pi\|g_{\lambda,\tau})<\infty.
$$
Then
$$
\|V_{\nu_\star}-V^\pi\|_{L^2(d^\pi)}^2
\leq
\frac{\tau}{1-\gamma}
\left[
D_{\mathrm{KL}}(\nu^\pi\|g_{\lambda,\tau})
-
D_{\mathrm{KL}}(\nu_\star\|g_{\lambda,\tau})
\right]
\leq
\frac{\tau}{1-\gamma}
D_{\mathrm{KL}}(\nu^\pi\|g_{\lambda,\tau}).
$$
\end{theorem}

Theorem~\ref{thm:policy_evaluation_error} separates two sources of
error. The first term measures the approximation error of the
mean-field function class, while the second reflects the parameter-space
regularization. In the realizable case, the approximation term
vanishes.

\subsection{Extension to Soft Q-Learning}

We now extend the analysis to entropy-regularized off-policy control.
Fix $\beta>0$ and a finite reference measure $m_{\mathcal A}$ on
$\mathcal A$, taken to be counting measure when $\mathcal A$ is finite.
For a Q-function $Q$, define the soft value operator
$$
(\mathcal S_\beta Q)(s)
:=
\beta\log
\int_{\mathcal A}
\exp\left(
\frac{Q(s,a)}{\beta}
\right)
m_{\mathcal A}(da)
$$
and the soft Bellman operator
$$
(\mathcal T_\beta Q)(s,a)
:=
r(s,a)
+
\gamma
\mathbb E_{s'\sim P(\cdot\mid s,a)}
\left[
(\mathcal S_\beta Q)(s')
\right].
$$

We write $Q_\nu(s,a):=Q(s,a;\nu)$ for the Q-function represented by
a parameter law $\nu$. For target and online laws $\mu$ and $\nu$,
define the frozen-target soft-Q objective
$$
H_Q(\mu,\nu)
:=
\frac12
\mathbb E_{(s,a,s')\sim\mu_b}
\left[
\left(
r(s,a)
+
\gamma(\mathcal S_\beta Q_\mu)(s')
-
Q_\nu(s,a)
\right)^2
\right],
$$
where
$\mu_b(ds,da,ds')=d_b(ds,da)P(ds'\mid s,a)$
for a fixed off-policy state-action sampling distribution $d_b$.
We set
$$
E_Q(\mu,\nu):=H_Q(\mu,\nu)+R(\nu),
\qquad
F_Q(\mu,\nu):=E_Q(\mu,\nu)-\tau\mathcal H(\nu).
$$

The dependence on the target and online laws remains analogous to that
in policy evaluation. Appendix~\ref{app:soft_q_mixed_smoothness}
verifies the corresponding mixed-smoothness condition with constant
$L_p^Q$. Let
$$
T_Q(\mu)
:=
\arg\min_\nu F_Q(\mu,\nu),
\qquad
G_Q(t)
:=
F_Q(\nu_t,\nu_t)
-
F_Q(\nu_t,T_Q(\nu_t)).
$$

\begin{theorem}[Exponential convergence of soft Q-learning]
\label{thm:soft_q_convergence}
Suppose that the soft-Q frozen-target objective satisfies the
mixed-smoothness condition with constant $L_p^Q$ and that its
associated Gibbs family satisfies the uniform log-Sobolev inequality
with constant $\rho$. Let $(\nu_t)_{t\geq0}$ be a sufficiently regular
solution to the corresponding noisy mean-field semi-gradient dynamics,
with $F_Q(\nu_0,\nu_0)<\infty$. If
$$
\rho\tau>L_p^Q,
$$
then the Gibbs response map $T_Q$ admits a unique fixed point
$\nu_\star^Q$, and the mean-field soft-Q dynamics converge
exponentially to $\nu_\star^Q$ in $W_2$. More precisely, for every
$t\geq0$,
$$
W_2(\nu_t,\nu_\star^Q)
\leq
\frac{\sqrt{2\rho\tau}}
{\rho\tau-L_p^Q}
\sqrt{G_Q(0)}
e^{-(\rho\tau-L_p^Q)t}.
$$
\end{theorem}

The proof follows the same two-step structure as in policy evaluation.
The soft-Q analogue of Lemma~\ref{lem:gap_decay} gives exponential
decay of $G_Q(t)$, while the analogue of Lemma~\ref{lem:response_contraction}
controls the Gibbs response map. We leave these parallel intermediate
statements inside the proof rather than repeating them as separate
lemmas.

To compare the limiting equilibrium with the soft Bellman fixed point, we assume full support of the off-policy state-action distribution. This converts the $L^2(d_b)$ Bellman residual to a sup-norm bound, after which we use the standard sup-norm contraction of the soft Bellman operator.

\begin{theorem}[Soft-Q approximation and regularization error]
\label{thm:soft_q_error}
Suppose that $\mathcal S$ and $\mathcal A$ are finite, and let $d_b$
be the state-action marginal of the off-policy sampling distribution.
Assume that
$$
d_{\min}
:=
\min_{s,a}d_b(s,a)
>
0.
$$
Let $Q_\beta^\star$ be the fixed point of the soft Bellman operator
$\mathcal T_\beta$, and let $\nu_\star^Q$ be a self-consistent soft-Q
equilibrium satisfying
$$
\nu_\star^Q
\in
\arg\min_\nu
F_Q(\nu_\star^Q,\nu).
$$
Then
$$
\left\|
Q_{\nu_\star^Q}
-
Q_\beta^\star
\right\|_\infty
\leq
\frac{1}
{(1-\gamma)\sqrt{d_{\min}}}
\inf_\nu
\left\{
\left\|
Q_\nu
-
\mathcal T_\beta Q_{\nu_\star^Q}
\right\|_{L^2(d_b)}
+
\sqrt{
\tau
D_{\mathrm{KL}}
(\nu\|g_{\lambda,\tau})
}
\right\},
$$
where the infimum is over parameter laws with finite
$D_{\mathrm{KL}}(\nu\|g_{\lambda,\tau})$.
\end{theorem}

\begin{figure}[t]
    \centering
    \includegraphics[width=\linewidth]{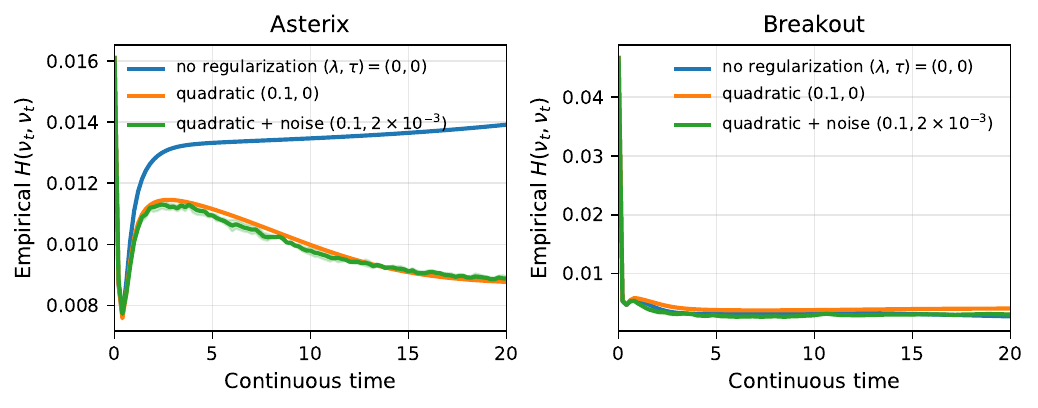}
    \caption{
Empirical $H(\nu_t,\nu_t)$ on Asterix and Breakout under
unregularized TD, quadratic regularization, and quadratic
regularization with Langevin noise. The quadratic-only baseline
isolates the effect of the confining regularizer from that of the
Langevin perturbation. Lower values correspond to smaller empirical
squared TD residuals. The noisy curves show the mean and one standard
deviation over five noise realizations.
}
    \label{fig:minatar}
\end{figure}

\section{Numerical Illustration}

We illustrate the effect of regularization on finite-width
semi-gradient TD in Asterix and Breakout (MinAtar) using fixed pretrained policies. We compare unregularized TD, quadratic
regularization, and quadratic regularization with Langevin noise. The quadratic-only baseline separates the effect of the confining regularization from the additional effect of Langevin noise. All methods use the same transition data and network initialization, so that the comparison isolates the effect of regularization.
Experimental details are given in Appendix~\ref{app:experiments}.

Figure~\ref{fig:minatar} shows two regimes. On Asterix, the
unregularized objective deteriorates after an initial decrease, while regularization suppresses this long-run increase. On Breakout, unregularized TD is already stable and the improvement is smaller. The additional effect of Langevin noise is modest at the scale considered here. The contrast between the two environments suggests that the benefit of regularization is most visible when the unregularized semi-gradient dynamics exhibit unfavorable long-run behavior. These experiments provide a qualitative illustration of the stability mechanism. We do not estimate $\rho$ or $L_p$ in these environments, so they are not a direct verification of the sufficient condition $\rho\tau>L_p$.

\section{Conclusion}

We studied noisy semi-gradient TD with two-layer neural networks in the mean-field regime and identified a stability mechanism in which relaxation toward the frozen-target minimizer dominates the sensitivity of the moving bootstrap target. Under this condition, we proved exponential convergence to a unique self-consistent equilibrium and quantified its error relative to the unregularized value function. We also showed that the same mechanism extends to entropy-regularized off-policy soft Q-learning.

\paragraph{Scope and limitations.}
Our analysis focuses on the population continuous-time dynamics in order to isolate the stability mechanism caused by the moving bootstrap target. We do not provide finite-width, time-discretization, or stochastic-gradient guarantees in this work. Related propagation-of-chaos results for mean-field TD \citep{agazzi:22} and finite-particle and discretization analyses for mean-field Langevin dynamics \citep{suzuki:23} provide natural starting points for such extensions. Establishing their quantitative counterparts for the present semi-gradient dynamics is left for future work.


\bibliography{main}
\bibliographystyle{rlj}

\clearpage
\appendix

\section{Verification of the Main Assumptions}
\label{app:assumption_verification}

\subsection{A Sufficient Condition for the Uniform Log-Sobolev Inequality}
\label{app:lsi_verification}

Assumption~\ref{ass:lsi} is imposed abstractly in the main results.
Here we give one concrete sufficient condition under an additional
second-order regularity assumption on the feature map. This additional
condition is used only to invoke the Bakry--\'Emery criterion and is
not specific to the semi-gradient structure studied in this paper. In
particular, it is satisfied by the two-layer parametrization in
Section~2.4 for standard smooth bounded activations such as $\tanh$ and
sigmoid when the input domain is bounded. Our main convergence theorem
is stated in terms of a general uniform log-Sobolev constant $\rho$, so
any sharper estimate of $\rho$ enlarges the resulting stability regime
without changing the argument.

Recall that, for fixed target and online laws $\mu,\nu$, the associated
Gibbs law has density
$$
\pi_{\mu,\nu}(\theta)
\propto
\exp\left(
-\frac{1}{\tau}
\frac{\delta E}{\delta\nu}(\mu,\nu)(\theta)
\right).
$$
For policy evaluation, define
$$
\Delta_{\mu,\nu}(s,a,s')
:=
r(s,a)
+
\gamma V(s';\mu)
-
V(s;\nu).
$$
Then
$$
\frac{\delta E}{\delta\nu}(\mu,\nu)(\theta)
=
-
\mathbb E_{s,a,s'}
\left[
\Delta_{\mu,\nu}(s,a,s')
\phi(s;\theta)
\right]
+
\frac{\lambda}{2}\|\theta\|^2.
$$

For the verification below, suppose in addition that
$\phi(s;\cdot)$ is twice continuously differentiable and that
\begin{equation}
B_0
:=
\sup_{s,\theta}
|\phi(s;\theta)|
<
\infty,
\qquad
B_2
:=
\sup_{s,\theta}
\|\nabla_\theta^2\phi(s;\theta)\|_{\mathrm{op}}
<
\infty.
\label{eq:feature_second_order_bounds}
\end{equation}

For the feature map
$$
\phi(x;\theta)
=
\sigma_1(c)\sigma_2(w^\top x+b),
\qquad
\theta=(c,w,b),
$$
these bounds hold when the input domain is bounded and
$\sigma_1,\sigma_2$ and their derivatives up to second order are
bounded. This includes, for example, $\tanh$ and sigmoid activations.

Since $\mu$ and $\nu$ are probability measures,
$$
|V(s;\mu)|
\leq
B_0,
\qquad
|V(s;\nu)|
\leq
B_0.
$$
Together with $r\in[0,1]$, this yields the uniform residual bound
$$
|\Delta_{\mu,\nu}(s,a,s')|
\leq
D_\Delta,
\qquad
D_\Delta
:=
1+(1+\gamma)B_0.
$$

\begin{proposition}[A Bakry--\'Emery certificate for Assumption~\ref{ass:lsi}]
\label{prop:lsi_sufficient}
Suppose that \eqref{eq:feature_second_order_bounds} holds and
$$
\lambda>D_\Delta B_2.
$$
Then every Gibbs law $\pi_{\mu,\nu}$ satisfies
$$
D_{\mathrm{KL}}(m\|\pi_{\mu,\nu})
\leq
\frac{1}{2\rho_{\mathrm{BE}}}
I(m\|\pi_{\mu,\nu})
$$
for every
$m\in\mathcal P_2^{\mathrm{ac}}(\mathbb R^d)$, uniformly over
$\mu,\nu$, with
$$
\rho_{\mathrm{BE}}
:=
\frac{\lambda-D_\Delta B_2}{\tau}.
$$
Equivalently,
$$
\rho_{\mathrm{BE}}\tau
=
\lambda-D_\Delta B_2.
$$
\end{proposition}

\begin{proof}
Differentiating the first variation with respect to $\theta$ twice
gives
$$
\nabla_\theta^2
\frac{\delta E}{\delta\nu}(\mu,\nu)(\theta)
=
\lambda I_d
-
\mathbb E_{s,a,s'}
\left[
\Delta_{\mu,\nu}(s,a,s')
\nabla_\theta^2\phi(s;\theta)
\right].
$$
For every unit vector $u\in\mathbb R^d$,
\begin{align}
u^\top
\nabla_\theta^2
\frac{\delta E}{\delta\nu}(\mu,\nu)(\theta)
u
&=
\lambda
-
\mathbb E_{s,a,s'}
\left[
\Delta_{\mu,\nu}(s,a,s')
u^\top\nabla_\theta^2\phi(s;\theta)u
\right]
\nonumber\\
&\geq
\lambda
-
\mathbb E_{s,a,s'}
\left[
|\Delta_{\mu,\nu}(s,a,s')|
\|\nabla_\theta^2\phi(s;\theta)\|_{\mathrm{op}}
\right]
\nonumber\\
&\geq
\lambda-D_\Delta B_2.
\label{eq:frozen_potential_curvature}
\end{align}
Hence
$$
\nabla_\theta^2
\frac{\delta E}{\delta\nu}(\mu,\nu)(\theta)
\succeq
(\lambda-D_\Delta B_2)I_d
$$
uniformly over $\mu,\nu,\theta$.

Up to an additive normalizing constant, the negative log-density of
$\pi_{\mu,\nu}$ is
$$
U_{\mu,\nu}(\theta)
:=
\frac{1}{\tau}
\frac{\delta E}{\delta\nu}(\mu,\nu)(\theta).
$$
Therefore,
$$
\nabla_\theta^2U_{\mu,\nu}(\theta)
\succeq
\frac{\lambda-D_\Delta B_2}{\tau}I_d
=
\rho_{\mathrm{BE}}I_d.
$$
The Bakry--\'Emery criterion then gives the claimed log-Sobolev
inequality.
\end{proof}

Proposition~\ref{prop:lsi_sufficient} is only one sufficient
certificate for Assumption~\ref{ass:lsi}. Combined with the stability
condition $\rho\tau>L_p$, it yields the explicit sufficient condition
$$
\lambda>D_\Delta B_2+L_p.
$$
The convergence analysis itself does not rely on this particular
certificate or on the additional second-order regularity used above.
Any stronger log-Sobolev estimate for the Gibbs family can instead be
substituted directly into the general condition $\rho\tau>L_p$.

\subsection{Verification of the Mixed-Smoothness Condition}
\label{app:mixed_smoothness}

We next verify the mixed-smoothness condition for the two-layer
mean-field parameterization used in the main text. For
$x\in\mathcal X$, define
$$
L_\phi(x)
:=
\sup_{\theta\in\mathbb R^d}
\|\nabla_\theta\phi(x;\theta)\|.
$$
We assume
$$
\mathbb E
\left[
L_\phi(s)L_\phi(s')
\right]
<
\infty
$$
under the transition distribution defining $H$. This condition follows
from the standing assumptions for the two-layer feature map, as shown
below.

The first variation of $H$ with respect to its target law is
$$
\frac{\delta H}{\delta\mu}(\mu,\nu)(\theta)
=
\gamma
\mathbb E_{s,a,s'}
\left[
\Delta_{\mu,\nu}(s,a,s')
\phi(s';\theta)
\right].
$$
Therefore, for any $\nu_1,\nu_2\in\mathcal P_2(\mathbb R^d)$,
\begin{align}
&
\nabla_\theta
\left[
\frac{\delta H}{\delta\mu}(\mu,\nu_1)(\theta)
-
\frac{\delta H}{\delta\mu}(\mu,\nu_2)(\theta)
\right]
\nonumber\\
&\qquad=
\gamma
\mathbb E_{s,a,s'}
\left[
\left(
V(s;\nu_2)-V(s;\nu_1)
\right)
\nabla_\theta\phi(s';\theta)
\right].
\label{eq:mixed_first_variation_difference}
\end{align}

For every fixed state $s$, the map
$\theta\mapsto\phi(s;\theta)$ is $L_\phi(s)$-Lipschitz. Hence,
by the Kantorovich--Rubinstein inequality and $W_1\leq W_2$,
\begin{equation}
|V(s;\nu_2)-V(s;\nu_1)|
=
\left|
\int
\phi(s;\vartheta)
(\nu_2-\nu_1)(d\vartheta)
\right|
\leq
L_\phi(s)W_2(\nu_1,\nu_2).
\label{eq:value_wasserstein_lipschitz}
\end{equation}
Substituting \eqref{eq:value_wasserstein_lipschitz} into
\eqref{eq:mixed_first_variation_difference} gives

\begin{align}
&
\left\|
\nabla_\theta
\left[
\frac{\delta H}{\delta\mu}(\mu,\nu_1)(\theta)
-
\frac{\delta H}{\delta\mu}(\mu,\nu_2)(\theta)
\right]
\right\|
\nonumber\\
&\qquad\leq
\gamma
\mathbb E_{s,a,s'}
\left[
L_\phi(s)L_\phi(s')
\right]
W_2(\nu_1,\nu_2).
\label{eq:mixed_smoothness_bound}
\end{align}
Thus Assumption~\ref{ass:mixed_smoothness} holds with
$$
L_p
=
\gamma
\mathbb E_{s,a,s'}
\left[
L_\phi(s)L_\phi(s')
\right].
$$

It remains to verify that this constant is finite for the feature map
used in Section~2.4. Recall that
$$
\phi(x;\theta)
=
\sigma_1(c)\sigma_2(w^\top x+b),
\qquad
\theta=(c,w,b).
$$
Let
$$
B_{1,0}:=\|\sigma_1\|_\infty,
\qquad
B_{1,1}:=\|\sigma_1'\|_\infty,
\qquad
B_{2,0}:=\|\sigma_2\|_\infty,
\qquad
B_{2,1}:=\|\sigma_2'\|_\infty.
$$
Writing $z=w^\top x+b$, we have
$$
\nabla_\theta\phi(x;\theta)
=
\begin{pmatrix}
\sigma_1'(c)\sigma_2(z)\\
\sigma_1(c)\sigma_2'(z)x\\
\sigma_1(c)\sigma_2'(z)
\end{pmatrix},
$$
and consequently
\begin{align}
L_\phi(x)^2
&\leq
B_{1,1}^2B_{2,0}^2
+
B_{1,0}^2B_{2,1}^2
\left(
1+\|x\|^2
\right)
\nonumber\\
&\leq
K_\phi^2
\left(
1+\|x\|^2
\right),
\label{eq:two_layer_feature_gradient_bound}
\end{align}
where
$$
K_\phi^2
:=
B_{1,1}^2B_{2,0}^2
+
B_{1,0}^2B_{2,1}^2.
$$
Therefore, by Cauchy--Schwarz,
\begin{align}
L_p
&\leq
\gamma
\sqrt{
\mathbb E[L_\phi(s)^2]
\mathbb E[L_\phi(s')^2]
}
\nonumber\\
&\leq
\gamma K_\phi^2
\sqrt{
\left(
1+\mathbb E\|s\|^2
\right)
\left(
1+\mathbb E\|s'\|^2
\right)
}.
\label{eq:explicit_two_layer_Lp_bound}
\end{align}
Hence $L_p<\infty$ whenever the state marginals of the transition
distribution have finite second moments. In particular, this is
automatic under the bounded-state assumption in Section~2.2.

\section{Proofs for Policy Evaluation}
\label{app:policy_evaluation}

\subsection{Proof of Lemma~\ref{lem:gap_decay}}
\label{app:proof_gap_decay}

\begin{proof}
Recall that
$$
G(t)
=
F(\nu_t,\nu_t)
-
F(\nu_t,T(\nu_t)),
$$
where $\nu_t(d\theta)=q_t(\theta)d\theta$ and $T(\nu_t)$ is the
minimizer of $\nu\mapsto F(\nu_t,\nu)$.
By the chain rule and the envelope theorem applied to the minimizer
$T(\nu_t)$,
\begin{align}
\frac{d}{dt}G(t)
&=
\int
\partial_t q_t(\theta)
\left[
\frac{\delta F}{\delta\nu}(\nu_t,\nu_t)(\theta)
+
\frac{\delta E}{\delta\mu}(\nu_t,\nu_t)(\theta)
-
\frac{\delta E}{\delta\mu}
(\nu_t,T(\nu_t))(\theta)
\right]
d\theta.
\label{eq:gap_chain_rule}
\end{align}
Here the derivative of $F$ with respect to its first argument is the
same as that of $E$, since the entropy and quadratic regularization
act only on the second argument.

To separate the two contributions to the evolution of $G(t)$, define
the relaxation field
$$
v_t^{\mathrm{rel}}(\theta)
:=
\nabla_\theta
\frac{\delta F}{\delta\nu}
(\nu_t,\nu_t)(\theta)
$$
and the moving-target field
$$
v_t^{\mathrm{mov}}(\theta)
:=
\nabla_\theta
\left[
\frac{\delta E}{\delta\mu}
(\nu_t,\nu_t)(\theta)
-
\frac{\delta E}{\delta\mu}
(\nu_t,T(\nu_t))(\theta)
\right].
$$
Substituting the Fokker--Planck equation
$$
\partial_t q_t
=
\nabla_\theta\cdot
\left(
q_t v_t^{\mathrm{rel}}
\right)
$$
into \eqref{eq:gap_chain_rule} and integrating by parts gives
\begin{align}
\frac{d}{dt}G(t)
=
-
\left\|
v_t^{\mathrm{rel}}
\right\|_{L^2(\nu_t)}^2
-
\left\langle
v_t^{\mathrm{mov}},
v_t^{\mathrm{rel}}
\right\rangle_{L^2(\nu_t)}.
\label{eq:gap_field_decomposition}
\end{align}

We first control the relaxation field.
By definition of the Gibbs law,
$$
\pi_{\nu_t,\nu_t}(\theta)
\propto
\exp\left(
-\frac{1}{\tau}
\frac{\delta E}{\delta\nu}
(\nu_t,\nu_t)(\theta)
\right).
$$
Since
$F(\mu,\nu)=E(\mu,\nu)-\tau\mathcal H(\nu)$,
we obtain
$$
v_t^{\mathrm{rel}}(\theta)
=
\tau
\nabla_\theta
\log
\frac{q_t(\theta)}
{\pi_{\nu_t,\nu_t}(\theta)}.
$$
Consequently,
$$
\left\|
v_t^{\mathrm{rel}}
\right\|_{L^2(\nu_t)}^2
=
\tau^2
I(\nu_t\|\pi_{\nu_t,\nu_t}).
$$
Assumption~\ref{ass:lsi} and the upper bound in the
entropy-sandwich inequality imply
\begin{align}
\left\|
v_t^{\mathrm{rel}}
\right\|_{L^2(\nu_t)}^2
&\geq
2\rho\tau^2
D_{\mathrm{KL}}
(\nu_t\|\pi_{\nu_t,\nu_t})
\nonumber\\
&\geq
2\rho\tau
\left[
F(\nu_t,\nu_t)
-
F(\nu_t,T(\nu_t))
\right]
\nonumber\\
&=
2\rho\tau G(t).
\label{eq:relaxation_field_lower_bound}
\end{align}

We next control the moving-target field.
Since
$E(\mu,\nu)=H(\mu,\nu)+R(\nu)$,
differentiation with respect to the first argument gives
$$
\frac{\delta E}{\delta\mu}
=
\frac{\delta H}{\delta\mu}.
$$
Assumption~\ref{ass:mixed_smoothness} therefore gives, pointwise in
$\theta$,
$$
\left\|
v_t^{\mathrm{mov}}(\theta)
\right\|
\leq
L_p
W_2(\nu_t,T(\nu_t)).
$$
Hence
\begin{align}
\left\|
v_t^{\mathrm{mov}}
\right\|_{L^2(\nu_t)}^2
\leq
L_p^2
W_2^2(\nu_t,T(\nu_t)).
\label{eq:moving_field_wasserstein}
\end{align}
The log-Sobolev inequality for $T(\nu_t)$ implies the corresponding
Talagrand inequality,
$$
W_2^2(\nu_t,T(\nu_t))
\leq
\frac{2}{\rho}
D_{\mathrm{KL}}
(\nu_t\|T(\nu_t)).
$$
Using the lower bound in the entropy-sandwich inequality,
$$
\tau
D_{\mathrm{KL}}
(\nu_t\|T(\nu_t))
\leq
G(t),
$$
we conclude from \eqref{eq:moving_field_wasserstein} that
\begin{align}
\left\|
v_t^{\mathrm{mov}}
\right\|_{L^2(\nu_t)}^2
\leq
\frac{2L_p^2}{\rho\tau}
G(t).
\label{eq:moving_field_upper_bound}
\end{align}

Combining
\eqref{eq:relaxation_field_lower_bound} and
\eqref{eq:moving_field_upper_bound} gives
$$
\begin{aligned}
\left\|
v_t^{\mathrm{mov}}
\right\|_{L^2(\nu_t)}
&\leq
L_p
\sqrt{
\frac{2G(t)}{\rho\tau}
}
\\
&\leq
\frac{L_p}{\rho\tau}
\left\|
v_t^{\mathrm{rel}}
\right\|_{L^2(\nu_t)}.
\end{aligned}
$$
Therefore, by \eqref{eq:gap_field_decomposition} and
Cauchy--Schwarz,
$$
\begin{aligned}
\frac{d}{dt}G(t)
&\leq
-
\left\|
v_t^{\mathrm{rel}}
\right\|_{L^2(\nu_t)}^2
+
\left\|
v_t^{\mathrm{mov}}
\right\|_{L^2(\nu_t)}
\left\|
v_t^{\mathrm{rel}}
\right\|_{L^2(\nu_t)}
\\
&\leq
-
\left(
1-\frac{L_p}{\rho\tau}
\right)
\left\|
v_t^{\mathrm{rel}}
\right\|_{L^2(\nu_t)}^2.
\end{aligned}
$$
Since $\rho\tau>L_p$, the coefficient is positive.
Applying \eqref{eq:relaxation_field_lower_bound} once more yields
$$
\frac{d}{dt}G(t)
\leq
-
2(\rho\tau-L_p)G(t).
$$
Gronwall's inequality therefore gives
$$
G(t)
\leq
e^{-2(\rho\tau-L_p)t}G(0),
$$
as claimed.
\end{proof}

\subsection{Proof of Lemma~\ref{lem:response_contraction}}
\label{app:proof_response_contraction}

\begin{proof}
Fix $\mu_1,\mu_2\in\mathcal P_2(\mathbb R^d)$ and write
$$
\nu_1:=T(\mu_1),
\qquad
\nu_2:=T(\mu_2).
$$
Applying the lower bound in the entropy-sandwich inequality to the
frozen-target objective $F(\mu_2,\cdot)$ gives
$$
\tau D_{\mathrm{KL}}(\nu_1\|\nu_2)
\leq
F(\mu_2,\nu_1)-F(\mu_2,\nu_2).
$$
Exchanging the roles of $\mu_1$ and $\mu_2$ similarly gives
$$
\tau D_{\mathrm{KL}}(\nu_2\|\nu_1)
\leq
F(\mu_1,\nu_2)-F(\mu_1,\nu_1).
$$
Adding the two inequalities, the quadratic regularization and entropy
terms cancel because they depend only on the second argument. Hence
\begin{align}
\tau
\left[
D_{\mathrm{KL}}(\nu_1\|\nu_2)
+
D_{\mathrm{KL}}(\nu_2\|\nu_1)
\right]
\leq
\Delta_H,
\label{eq:response_symmetric_kl}
\end{align}
where
\begin{align}
\Delta_H
:={}&
H(\mu_2,\nu_1)-H(\mu_1,\nu_1)
\nonumber\\
&-
H(\mu_2,\nu_2)+H(\mu_1,\nu_2).
\label{eq:response_delta_H}
\end{align}

Let
$$
\mu_s:=(1-s)\mu_1+s\mu_2,
\qquad s\in[0,1].
$$
By the first-variation formula,
\begin{align}
\Delta_H
=
\int_0^1
\int_{\mathbb R^d}
\bigg[
\frac{\delta H}{\delta\mu}(\mu_s,\nu_1)(\theta)
-
\frac{\delta H}{\delta\mu}(\mu_s,\nu_2)(\theta)
\bigg]
(\mu_2-\mu_1)(d\theta)
\,\mathrm ds.
\label{eq:response_first_variation}
\end{align}
Assumption~\ref{ass:mixed_smoothness} implies that, for every
$s\in[0,1]$, the function inside the square brackets is
$L_pW_2(\nu_1,\nu_2)$-Lipschitz in $\theta$. Therefore,
Kantorovich--Rubinstein duality and $W_1\leq W_2$ yield
\begin{align}
\Delta_H
\leq
L_p
W_2(\mu_1,\mu_2)
W_2(\nu_1,\nu_2).
\label{eq:response_mixed_bound}
\end{align}

On the other hand, Assumption~\ref{ass:lsi} implies Talagrand's
inequality for both $\nu_1$ and $\nu_2$. Applying it in both
directions and adding gives
\begin{align}
\rho W_2^2(\nu_1,\nu_2)
\leq
D_{\mathrm{KL}}(\nu_1\|\nu_2)
+
D_{\mathrm{KL}}(\nu_2\|\nu_1).
\label{eq:response_talagrand}
\end{align}
Combining
\eqref{eq:response_symmetric_kl},
\eqref{eq:response_mixed_bound}, and
\eqref{eq:response_talagrand}, we obtain
\begin{align}
\rho\tau W_2^2(\nu_1,\nu_2)
\leq
L_p
W_2(\mu_1,\mu_2)
W_2(\nu_1,\nu_2).
\label{eq:response_contraction_predivision}
\end{align}
If $W_2(\nu_1,\nu_2)=0$, the desired inequality is immediate.
Otherwise, dividing
\eqref{eq:response_contraction_predivision} by
$W_2(\nu_1,\nu_2)$ gives
$$
W_2(T(\mu_1),T(\mu_2))
\leq
\frac{L_p}{\rho\tau}
W_2(\mu_1,\mu_2).
$$
In particular, $T$ is a contraction whenever $\rho\tau>L_p$.
\end{proof}

\subsection{Proof of Theorem~\ref{thm:policy_evaluation_convergence}}
\label{app:proof_policy_convergence}

\begin{proof}
Under $\rho\tau>L_p$, Lemma~\ref{lem:response_contraction} and
completeness of
$(\mathcal P_2(\mathbb R^d),W_2)$ imply that $T$ has a unique fixed
point $\nu_\star$, satisfying
$$
T(\nu_\star)=\nu_\star.
$$
By the triangle inequality, this fixed-point identity, and
Lemma~\ref{lem:response_contraction},
\begin{align}
W_2(\nu_t,\nu_\star)
&\leq
W_2(\nu_t,T(\nu_t))
+
W_2(T(\nu_t),T(\nu_\star))
\nonumber\\
&\leq
W_2(\nu_t,T(\nu_t))
+
\frac{L_p}{\rho\tau}
W_2(\nu_t,\nu_\star).
\end{align}
Since $\rho\tau>L_p$,
\begin{align}
W_2(\nu_t,\nu_\star)
\leq
\frac{\rho\tau}{\rho\tau-L_p}
W_2(\nu_t,T(\nu_t)).
\label{eq:equilibrium_reduction_to_response}
\end{align}

By Talagrand's inequality for the Gibbs law $T(\nu_t)$ and the lower
bound in the entropy-sandwich inequality,
\begin{align}
W_2^2(\nu_t,T(\nu_t))
&\leq
\frac{2}{\rho}
D_{\mathrm{KL}}(\nu_t\|T(\nu_t))
\nonumber\\
&\leq
\frac{2}{\rho\tau}G(t).
\label{eq:distance_to_response}
\end{align}
Combining
\eqref{eq:equilibrium_reduction_to_response} and
\eqref{eq:distance_to_response} gives
\begin{align}
W_2(\nu_t,\nu_\star)
\leq
\frac{\sqrt{2\rho\tau}}{\rho\tau-L_p}
\sqrt{G(t)}.
\label{eq:equilibrium_distance_gap}
\end{align}
Finally, Lemma~\ref{lem:gap_decay} gives
$$
G(t)
\leq
e^{-2(\rho\tau-L_p)t}G(0).
$$
Substituting this estimate into
\eqref{eq:equilibrium_distance_gap} yields
$$
W_2(\nu_t,\nu_\star)
\leq
\frac{\sqrt{2\rho\tau}}{\rho\tau-L_p}
\sqrt{G(0)}
e^{-(\rho\tau-L_p)t},
$$
which proves the claim.
\end{proof}

\subsection{Proof of Theorem~\ref{thm:policy_evaluation_error}}
\label{app:proof_policy_error}

\begin{proof}
Fix any comparator
$\nu\in\mathcal P_2(\mathbb R^d)$ satisfying
$D_{\mathrm{KL}}(\nu\|g_{\lambda,\tau})<\infty$.
Since $\nu_\star$ minimizes
$\nu'\mapsto F(\nu_\star,\nu')$, consider the interpolation
$$
\nu_\varepsilon
=
(1-\varepsilon)\nu_\star+\varepsilon\nu,
\qquad
\varepsilon\in(0,1].
$$
The optimality of $\nu_\star$ gives
$$
F(\nu_\star,\nu_\varepsilon)
-
F(\nu_\star,\nu_\star)
\geq0.
$$
Using the identity
$$
R(\nu)-\tau\mathcal H(\nu)
=
\tau D_{\mathrm{KL}}(\nu\|g_{\lambda,\tau})
-
\frac{\tau d}{2}
\log\frac{2\pi\tau}{\lambda},
$$
together with the convexity of relative entropy, we obtain the
corresponding first-order inequality as
$\varepsilon\downarrow0$. Under the on-policy transition distribution
specified in Theorem~\ref{thm:policy_evaluation_error}, the regression
term satisfies
$$
\begin{aligned}
&
\mathbb E
\left[
\left(
V_{\nu_\star}(s)
-
r(s,a)
-
\gamma V_{\nu_\star}(s')
\right)
\left(
V_\nu(s)-V_{\nu_\star}(s)
\right)
\right]
\\
&\qquad=
\left\langle
V_{\nu_\star}-T^\pi V_{\nu_\star},
V_\nu-V_{\nu_\star}
\right\rangle_{L^2(d^\pi)}.
\end{aligned}
$$
Hence,

\begin{align}
&
\left\langle
V_{\nu_\star}-T^\pi V_{\nu_\star},
V_\nu-V_{\nu_\star}
\right\rangle_{L^2(d^\pi)}
\nonumber\\
&\qquad
+
\tau
\left[
D_{\mathrm{KL}}(\nu\|g_{\lambda,\tau})
-
D_{\mathrm{KL}}(\nu_\star\|g_{\lambda,\tau})
\right]
\geq0.
\label{eq:policy_error_optimality}
\end{align}

Since $V^\pi=T^\pi V^\pi$,
\begin{align}
V_{\nu_\star}-T^\pi V_{\nu_\star}
=
(I-\gamma P^\pi)
(V_{\nu_\star}-V^\pi).
\label{eq:policy_error_bellman_residual}
\end{align}
Moreover,
$$
V_\nu-V_{\nu_\star}
=
(V_\nu-V^\pi)
-
(V_{\nu_\star}-V^\pi).
$$
Substituting these identities into
\eqref{eq:policy_error_optimality} gives
\begin{align}
&
\left\langle
(I-\gamma P^\pi)(V_{\nu_\star}-V^\pi),
V_{\nu_\star}-V^\pi
\right\rangle_{d^\pi}
\nonumber\\
&\qquad
\leq
\left\langle
(I-\gamma P^\pi)(V_{\nu_\star}-V^\pi),
V_\nu-V^\pi
\right\rangle_{d^\pi}
+
\tau\Delta_{\mathrm{KL}}(\nu),
\label{eq:policy_error_basic}
\end{align}
where
$$
\Delta_{\mathrm{KL}}(\nu)
=
D_{\mathrm{KL}}(\nu\|g_{\lambda,\tau})
-
D_{\mathrm{KL}}(\nu_\star\|g_{\lambda,\tau}).
$$

Because $d^\pi$ is stationary for $P^\pi$, the Markov operator
$P^\pi$ is nonexpansive in $L^2(d^\pi)$,
\begin{align}
\|P^\pi f\|_{L^2(d^\pi)}
\leq
\|f\|_{L^2(d^\pi)}.
\label{eq:policy_markov_nonexpansive}
\end{align}
Indeed, Jensen's inequality and stationarity give
$$
\|P^\pi f\|_{L^2(d^\pi)}^2
\leq
\int P^\pi(f^2)(s)\,d^\pi(ds)
=
\|f\|_{L^2(d^\pi)}^2.
$$
It follows that
\begin{align}
&
\left\langle
(I-\gamma P^\pi)(V_{\nu_\star}-V^\pi),
V_{\nu_\star}-V^\pi
\right\rangle_{d^\pi}
\nonumber\\
&\qquad
\geq
(1-\gamma)
\|V_{\nu_\star}-V^\pi\|_{L^2(d^\pi)}^2,
\label{eq:policy_error_coercivity}
\end{align}
and
\begin{align}
\left\|
(I-\gamma P^\pi)(V_{\nu_\star}-V^\pi)
\right\|_{L^2(d^\pi)}
\leq
(1+\gamma)
\|V_{\nu_\star}-V^\pi\|_{L^2(d^\pi)}.
\label{eq:policy_error_residual_upper}
\end{align}
Therefore, by Cauchy--Schwarz,
\eqref{eq:policy_error_basic} implies
\begin{align}
(1-\gamma)
\|V_{\nu_\star}-V^\pi\|_{L^2(d^\pi)}^2
&\leq
(1+\gamma)
\|V_{\nu_\star}-V^\pi\|_{L^2(d^\pi)}
\|V_\nu-V^\pi\|_{L^2(d^\pi)}
\nonumber\\
&\qquad
+
\tau\Delta_{\mathrm{KL}}(\nu).
\label{eq:policy_error_master}
\end{align}

Dropping the nonnegative term
$D_{\mathrm{KL}}(\nu_\star\|g_{\lambda,\tau})$ from the right-hand
side of \eqref{eq:policy_error_master}, we obtain
$$
\begin{aligned}
\|V_{\nu_\star}-V^\pi\|_{L^2(d^\pi)}^2
&\leq
\frac{1+\gamma}{1-\gamma}
\|V_\nu-V^\pi\|_{L^2(d^\pi)}
\|V_{\nu_\star}-V^\pi\|_{L^2(d^\pi)}
\\
&\qquad
+
\frac{\tau}{1-\gamma}
D_{\mathrm{KL}}(\nu\|g_{\lambda,\tau}).
\end{aligned}
$$
Using the elementary implication
$$
x^2\leq bx+c
\quad\Longrightarrow\quad
x\leq b+\sqrt c
\qquad
(x,b,c\geq0),
$$
we conclude that
\begin{align}
\|V_{\nu_\star}-V^\pi\|_{L^2(d^\pi)}
\leq
\frac{1+\gamma}{1-\gamma}
\|V_\nu-V^\pi\|_{L^2(d^\pi)}
+
\sqrt{
\frac{\tau}{1-\gamma}
D_{\mathrm{KL}}(\nu\|g_{\lambda,\tau})
}.
\end{align}
Taking the infimum over all admissible comparators $\nu$ proves the
first claim.

For the realizable case, choose $\nu=\nu^\pi$ such that
$V_{\nu^\pi}=V^\pi$. Then
\eqref{eq:policy_error_master} gives directly
$$
(1-\gamma)
\|V_{\nu_\star}-V^\pi\|_{L^2(d^\pi)}^2
\leq
\tau
\left[
D_{\mathrm{KL}}(\nu^\pi\|g_{\lambda,\tau})
-
D_{\mathrm{KL}}(\nu_\star\|g_{\lambda,\tau})
\right].
$$
Dividing by $1-\gamma$ yields
$$
\|V_{\nu_\star}-V^\pi\|_{L^2(d^\pi)}^2
\leq
\frac{\tau}{1-\gamma}
\left[
D_{\mathrm{KL}}(\nu^\pi\|g_{\lambda,\tau})
-
D_{\mathrm{KL}}(\nu_\star\|g_{\lambda,\tau})
\right].
$$
Finally,
$D_{\mathrm{KL}}(\nu_\star\|g_{\lambda,\tau})\geq0$ gives the
second inequality in the theorem.
\end{proof}

\section{Soft Q-Learning}
\label{app:soft_q}

\subsection{Frozen-Target Objective and Mixed Smoothness}
\label{app:soft_q_mixed_smoothness}

For soft Q-learning, we take the input space to be
$\mathcal X=\mathcal S\times\mathcal A$ and write
$$
Q_\nu(s,a)
:=
\int_{\mathbb R^d}
\phi((s,a);\theta)\nu(d\theta).
$$
Recall that the frozen-target soft-Q objective is
$$
H_Q(\mu,\nu)
=
\frac12
\mathbb E_{(s,a,s')\sim\mu_b}
\left[
\left(
r(s,a)
+
\gamma(\mathcal S_\beta Q_\mu)(s')
-
Q_\nu(s,a)
\right)^2
\right],
$$
where
$$
\mu_b(ds,da,ds')
=
d_b(ds,da)P(ds'\mid s,a).
$$
We set
$$
E_Q(\mu,\nu)
=
H_Q(\mu,\nu)+R(\nu),
\qquad
F_Q(\mu,\nu)
=
E_Q(\mu,\nu)-\tau\mathcal H(\nu).
$$

We first verify that the dependence of $H_Q$ on the target and online
laws satisfies the same mixed-smoothness structure as in policy
evaluation. For each target law $\mu$, define the Boltzmann policy
induced by $Q_\mu$ as
$$
\pi_\mu^\beta(da'\mid s')
:=
\frac{
\exp\left(
Q_\mu(s',a')/\beta
\right)
m_{\mathcal A}(da')
}{
\displaystyle
\int_{\mathcal A}
\exp\left(
Q_\mu(s',\widetilde a)/\beta
\right)
m_{\mathcal A}(d\widetilde a)
}.
$$
By the definition of the soft value operator,
$$
(\mathcal S_\beta Q_\mu)(s')
=
\beta
\log
\int_{\mathcal A}
\exp\left(
\frac{Q_\mu(s',a')}{\beta}
\right)
m_{\mathcal A}(da').
$$
Since
$$
\frac{\delta Q_\mu(s',a')}
{\delta\mu}(\theta)
=
\phi((s',a');\theta),
$$
differentiating the soft value operator with respect to the target law
gives
\begin{align}
\frac{\delta(\mathcal S_\beta Q_\mu)(s')}
{\delta\mu}(\theta)
&=
\frac{
\displaystyle
\int_{\mathcal A}
\exp\left(
Q_\mu(s',a')/\beta
\right)
\phi((s',a');\theta)
m_{\mathcal A}(da')
}{
\displaystyle
\int_{\mathcal A}
\exp\left(
Q_\mu(s',\widetilde a)/\beta
\right)
m_{\mathcal A}(d\widetilde a)
}
\nonumber\\
&=
\int_{\mathcal A}
\phi((s',a');\theta)
\pi_\mu^\beta(da'\mid s').
\label{eq:soft_value_first_variation}
\end{align}

Define the soft-Q temporal-difference residual by
$$
\Delta_{\mu,\nu}^Q(s,a,s')
:=
r(s,a)
+
\gamma(\mathcal S_\beta Q_\mu)(s')
-
Q_\nu(s,a).
$$
The first variation of $H_Q$ with respect to the target law is then
\begin{align}
\frac{\delta H_Q}{\delta\mu}(\mu,\nu)(\theta)
=
\gamma
\mathbb E_{\mu_b}
\left[
\Delta_{\mu,\nu}^Q(s,a,s')
\int_{\mathcal A}
\phi((s',a');\theta)
\pi_\mu^\beta(da'\mid s')
\right].
\label{eq:soft_q_target_first_variation}
\end{align}

For fixed $\mu$, both
$(\mathcal S_\beta Q_\mu)(s')$ and
$\pi_\mu^\beta(\cdot\mid s')$ are fixed.
Hence, for two online laws $\nu_1,\nu_2$,
\begin{align}
&
\nabla_\theta
\left[
\frac{\delta H_Q}{\delta\mu}(\mu,\nu_1)(\theta)
-
\frac{\delta H_Q}{\delta\mu}(\mu,\nu_2)(\theta)
\right]
\nonumber\\
&\qquad=
\gamma
\mathbb E_{\mu_b}
\left[
\left(
Q_{\nu_2}(s,a)-Q_{\nu_1}(s,a)
\right)
\int_{\mathcal A}
\nabla_\theta
\phi((s',a');\theta)
\pi_\mu^\beta(da'\mid s')
\right].
\label{eq:soft_q_mixed_difference}
\end{align}

For $x=(s,a)$, define
$$
L_\phi(x)
:=
\sup_{\theta\in\mathbb R^d}
\|\nabla_\theta\phi(x;\theta)\|.
$$
Under the bounded state-action assumption and the bounded first
derivatives of the activations used in Section~\ref{subsec:two_layer},
$$
\overline L_\phi
:=
\sup_{x\in\mathcal S\times\mathcal A}
L_\phi(x)
<
\infty.
$$
As in policy evaluation, Kantorovich--Rubinstein duality gives
$$
\left|
Q_{\nu_2}(s,a)-Q_{\nu_1}(s,a)
\right|
\leq
L_\phi(s,a)
W_2(\nu_1,\nu_2)
\leq
\overline L_\phi
W_2(\nu_1,\nu_2).
$$
Moreover, because $\pi_\mu^\beta(\cdot\mid s')$ is a probability
measure,
$$
\left\|
\int_{\mathcal A}
\nabla_\theta
\phi((s',a');\theta)
\pi_\mu^\beta(da'\mid s')
\right\|
\leq
\overline L_\phi.
$$
Substituting these bounds into
\eqref{eq:soft_q_mixed_difference} yields
$$
\left\|
\nabla_\theta
\left[
\frac{\delta H_Q}{\delta\mu}(\mu,\nu_1)(\theta)
-
\frac{\delta H_Q}{\delta\mu}(\mu,\nu_2)(\theta)
\right]
\right\|
\leq
L_p^Q
W_2(\nu_1,\nu_2),
$$
where one may take
$$
L_p^Q
:=
\gamma\overline L_\phi^2.
$$
Thus the soft-Q frozen-target objective satisfies the same
mixed-smoothness condition used in the policy-evaluation analysis.

The corresponding noisy mean-field semi-gradient dynamics are
$$
\partial_t q_t
=
\nabla_\theta\cdot
\left(
q_t
\nabla_\theta
\frac{\delta F_Q}{\delta\nu}
(\nu_t,\nu_t)
\right),
\qquad
\nu_t(d\theta)=q_t(\theta)d\theta.
$$

\subsection{Proof of Theorem~\ref{thm:soft_q_convergence}}
\label{app:proof_soft_q_convergence}

\begin{proof}
By Appendix~\ref{app:soft_q_mixed_smoothness}, the frozen-target
soft-Q objective satisfies the same mixed-smoothness condition as in
policy evaluation, with constant $L_p^Q$. Moreover, for each fixed
target law $\mu$, the Gibbs characterization, the uniform
log-Sobolev inequality, and the entropy-sandwich inequality apply to
$F_Q(\mu,\cdot)$ exactly as they do to $F(\mu,\cdot)$.

Therefore, the proof of Lemma~\ref{lem:gap_decay} applies with
$$
F,\ E,\ H,\ T,\ L_p,\ G
\quad\text{replaced by}\quad
F_Q,\ E_Q,\ H_Q,\ T_Q,\ L_p^Q,\ G_Q.
$$
Hence
\begin{equation}
\frac{\mathrm d}{\mathrm dt}G_Q(t)
\leq
-2(\rho\tau-L_p^Q)G_Q(t),
\label{eq:soft_q_gap_decay_appendix}
\end{equation}
and consequently
\begin{equation}
G_Q(t)
\leq
e^{-2(\rho\tau-L_p^Q)t}G_Q(0).
\label{eq:soft_q_gap_decay_exp_appendix}
\end{equation}

Likewise, the proof of
Lemma~\ref{lem:response_contraction} applies to the soft-Q Gibbs
response map and gives
\begin{equation}
W_2(T_Q(\mu_1),T_Q(\mu_2))
\leq
\frac{L_p^Q}{\rho\tau}
W_2(\mu_1,\mu_2).
\label{eq:soft_q_response_contraction_appendix}
\end{equation}
Thus, if $\rho\tau>L_p^Q$, the map $T_Q$ is a contraction on
$\mathcal P_2(\mathbb R^d)$ and admits a unique fixed point
$\nu_\star^Q$.

Finally, repeating the argument in the proof of
Theorem~\ref{thm:policy_evaluation_convergence},
\begin{align}
W_2(\nu_t,\nu_\star^Q)
&\leq
\frac{\rho\tau}{\rho\tau-L_p^Q}
W_2(\nu_t,T_Q(\nu_t))
\nonumber\\
&\leq
\frac{\sqrt{2\rho\tau}}{\rho\tau-L_p^Q}
\sqrt{G_Q(t)}.
\end{align}
Combining this inequality with
\eqref{eq:soft_q_gap_decay_exp_appendix} yields
$$
W_2(\nu_t,\nu_\star^Q)
\leq
\frac{\sqrt{2\rho\tau}}{\rho\tau-L_p^Q}
\sqrt{G_Q(0)}
e^{-(\rho\tau-L_p^Q)t},
$$
which proves the theorem.
\end{proof}

\subsection{Proof of Theorem~\ref{thm:soft_q_error}}

\begin{proof}
Fix any comparator $\nu$ satisfying
$D_{\mathrm{KL}}(\nu\|g_{\lambda,\tau})<\infty$.
Since
$$
\nu_\star^Q
\in
\arg\min_{\nu'}
F_Q(\nu_\star^Q,\nu'),
$$
the same interpolation and convexity argument used in the proof of
Theorem~\ref{thm:policy_evaluation_error} gives
$$
\left\langle
Q_{\nu_\star^Q}-\mathcal{T}_\beta Q_{\nu_\star^Q},\,
Q_\nu-Q_{\nu_\star^Q}
\right\rangle_{L^2(d_b)}
+
\tau
\left[
D_{\mathrm{KL}}(\nu\|g_{\lambda,\tau})
-
D_{\mathrm{KL}}(\nu_\star^Q\|g_{\lambda,\tau})
\right]
\ge 0.
$$

Using
$$
Q_\nu-Q_{\nu_\star^Q}
=
\left(
Q_\nu-\mathcal{T}_\beta Q_{\nu_\star^Q}
\right)
-
\left(
Q_{\nu_\star^Q}-\mathcal{T}_\beta Q_{\nu_\star^Q}
\right),
$$
we obtain
$$
\begin{aligned}
\left\|
Q_{\nu_\star^Q}-\mathcal{T}_\beta Q_{\nu_\star^Q}
\right\|_{L^2(d_b)}^2
&\le
\left\langle
Q_{\nu_\star^Q}-\mathcal{T}_\beta Q_{\nu_\star^Q},\,
Q_\nu-\mathcal{T}_\beta Q_{\nu_\star^Q}
\right\rangle_{L^2(d_b)}
\\
&\quad+
\tau
\left[
D_{\mathrm{KL}}(\nu\|g_{\lambda,\tau})
-
D_{\mathrm{KL}}(\nu_\star^Q\|g_{\lambda,\tau})
\right].
\end{aligned}
$$
By Cauchy--Schwarz and the nonnegativity of relative entropy,
$$
\begin{aligned}
\left\|
Q_{\nu_\star^Q}-\mathcal{T}_\beta Q_{\nu_\star^Q}
\right\|_{L^2(d_b)}^2
&\le
\left\|
Q_{\nu_\star^Q}-\mathcal{T}_\beta Q_{\nu_\star^Q}
\right\|_{L^2(d_b)}
\left\|
Q_\nu-\mathcal{T}_\beta Q_{\nu_\star^Q}
\right\|_{L^2(d_b)}
\\
&\quad+
\tau D_{\mathrm{KL}}(\nu\|g_{\lambda,\tau}).
\end{aligned}
$$
Using the elementary implication
$$
x^2\le bx+c
\quad\Longrightarrow\quad
x\le b+\sqrt{c},
\qquad
x,b,c\ge0,
$$
we conclude that
$$
\left\|
Q_{\nu_\star^Q}-\mathcal{T}_\beta Q_{\nu_\star^Q}
\right\|_{L^2(d_b)}
\le
\left\|
Q_\nu-\mathcal{T}_\beta Q_{\nu_\star^Q}
\right\|_{L^2(d_b)}
+
\sqrt{
\tau D_{\mathrm{KL}}(\nu\|g_{\lambda,\tau})
}.
$$

The soft Bellman operator is a $\gamma$-contraction in the sup norm,
$$
\left\|
\mathcal{T}_\beta Q-\mathcal{T}_\beta Q'
\right\|_\infty
\le
\gamma
\left\|
Q-Q'
\right\|_\infty.
$$
Since
$Q_\beta^\star=\mathcal{T}_\beta Q_\beta^\star$,
we have
$$
\begin{aligned}
\left\|
Q_{\nu_\star^Q}-Q_\beta^\star
\right\|_\infty
&\le
\left\|
Q_{\nu_\star^Q}
-
\mathcal{T}_\beta Q_{\nu_\star^Q}
\right\|_\infty
+
\left\|
\mathcal{T}_\beta Q_{\nu_\star^Q}
-
\mathcal{T}_\beta Q_\beta^\star
\right\|_\infty
\\
&\le
\left\|
Q_{\nu_\star^Q}
-
\mathcal{T}_\beta Q_{\nu_\star^Q}
\right\|_\infty
+
\gamma
\left\|
Q_{\nu_\star^Q}-Q_\beta^\star
\right\|_\infty.
\end{aligned}
$$
Therefore,
$$
\left\|
Q_{\nu_\star^Q}-Q_\beta^\star
\right\|_\infty
\le
\frac{1}{1-\gamma}
\left\|
Q_{\nu_\star^Q}
-
\mathcal{T}_\beta Q_{\nu_\star^Q}
\right\|_\infty.
$$

Because $d_{\min}>0$, every function $h$ on
$\mathcal{S}\times\mathcal{A}$ satisfies
$$
\|h\|_\infty
\le
\frac{1}{\sqrt{d_{\min}}}
\|h\|_{L^2(d_b)}.
$$
Hence,
$$
\begin{aligned}
\left\|
Q_{\nu_\star^Q}-Q_\beta^\star
\right\|_\infty
&\le
\frac{1}{(1-\gamma)\sqrt{d_{\min}}}
\left\|
Q_{\nu_\star^Q}
-
\mathcal{T}_\beta Q_{\nu_\star^Q}
\right\|_{L^2(d_b)}
\\
&\le
\frac{1}{(1-\gamma)\sqrt{d_{\min}}}
\left\{
\left\|
Q_\nu-\mathcal{T}_\beta Q_{\nu_\star^Q}
\right\|_{L^2(d_b)}
+
\sqrt{
\tau D_{\mathrm{KL}}(\nu\|g_{\lambda,\tau})
}
\right\}.
\end{aligned}
$$
Taking the infimum over all parameter laws $\nu$ with finite
$D_{\mathrm{KL}}(\nu\|g_{\lambda,\tau})$ proves the claim.
\end{proof}

\section{Experimental Details}
\label{app:experiments}

\paragraph{Environments and data collection.}
We consider policy evaluation on Asterix and Breakout in MinAtar.
For each environment, we use the pretrained PPO checkpoint distributed
with the \texttt{gymnax-blines} repository
(\texttt{agents/Asterix-MinAtar/ppo.pkl} and
\texttt{agents/Breakout-MinAtar/ppo.pkl}, respectively), and keep the
policy fixed throughout the experiment.
The repository provides the PPO training pipeline together with the
corresponding pretrained checkpoints in
\texttt{agents/<env\_name>/ppo.pkl}. 
Using these fixed stochastic policies, we collect a transition pool by
running $2048$ parallel environments for $128$ steps, yielding
$262{,}144$ transitions per environment.
Terminal transitions use the final observation before the environment
reset. From each pool, we use a fixed set of $16{,}384$ transitions for
training and a disjoint set of $65{,}536$ transitions for evaluation.
The same data split is used for all methods. These finite-rollout
distributions are used only for the numerical illustration and are not
intended to exactly represent a stationary state distribution.

\paragraph{Finite-width dynamics.}
We use the finite-particle approximation of the two-layer
parameterization in Section~2.4,
\[
V_M(s;\boldsymbol{\theta})
=
\frac{1}{M}\sum_{i=1}^M
\phi(s;\theta_i),
\qquad
\phi(s;\theta_i)
=
C\tanh(c_i)\tanh(w_i^\top s+b_i),
\]
with width $M=4096$ and output scale $C=100$.
Observations are flattened before being passed to the network.
We initialize
\[
w_i \sim \mathcal{N}(0,I)/\sqrt{d_{\mathrm{in}}},
\qquad
b_i \sim 0.1\,\mathcal{N}(0,1),
\qquad
c_i \sim \mathcal{N}(0,1),
\]
and use the same initialization for all compared methods.

At each iteration, the semi-gradient is computed on the entire fixed training set, with gradients stopped through the bootstrap target. We use discount factor $\gamma=0.99$, step size $\eta=10^{-4}$, and run for $2\times 10^5$ iterations, corresponding to continuous time $T=20$ under $t=k\eta$.
We compare
\[
(\lambda,\tau)\in
\left\{
(0,0),\,
(0.1,0),\,
(0.1,2\times10^{-3})
\right\},
\]
corresponding respectively to no regularization, quadratic
regularization only, and quadratic regularization with Langevin noise.
For $\tau>0$, the particle update includes independent Gaussian noise
with scale $\sqrt{2\tau\eta}$ as in the noisy semi-gradient dynamics.

\paragraph{Evaluation and randomness.}
We report the empirical counterpart of $H(\nu_t,\nu_t)$ on the fixed
evaluation set,
\[
\widehat H_{\mathrm{eval}}(\nu_t,\nu_t)
=
\frac{1}{2|\mathcal D_{\mathrm{eval}}|}
\sum_{(s,r,s')\in\mathcal D_{\mathrm{eval}}}
\left(
r+\gamma V(s';\nu_t)-V(s;\nu_t)
\right)^2.
\]
For $\tau=0$, the update is deterministic given the initialization and
dataset, and we therefore report a single trajectory.
For $\tau=2\times10^{-3}$, we keep the initialization fixed and vary
only the Langevin-noise realization. Figure~\ref{fig:minatar} reports
the mean and one standard deviation over five noise seeds.

\end{document}